\documentclass{article}
\usepackage{math, subfig}

\usepackage[nonatbib,preprint]{neurips_2019}

\usepackage[utf8]{inputenc} 
\usepackage[T1]{fontenc}    
\usepackage{hyperref}       
\usepackage{url}            
\usepackage{booktabs}       
\usepackage{amsfonts}       
\usepackage{nicefrac}       
\usepackage{microtype}      

\title{A cryptographic approach to black box adversarial machine learning}

\author{%
  Kevin Shi\\
  Columbia University\\
  \And
  Daniel Hsu\\
  Columbia University\\
  \And
  Allison Bishop\\
  Proof Trading\\
}

\begin{document}

\maketitle

\begin{abstract}
	We propose a new randomized ensemble technique with a provable security guarantee against black-box transfer attacks. Our proof constructs a new security problem for random binary classifiers which is easier to empirically verify and a reduction from the security of this new model to the security of the ensemble classifier. We provide experimental evidence of the security of our random binary classifiers, as well as empirical results of the adversarial accuracy of the overall ensemble to black-box attacks. Our construction crucially leverages hidden randomness in the multiclass-to-binary reduction.
\end{abstract}

\section{Introduction}
Current machine learning models are vulnerable at test time to adversarial examples, which are data points that have been imperceptibly modified from legitimate data points but are misclassified with high confidence. This phenomenon was first described by \cite{Szegedy2013}, who constructed a simple attack that resembled gradient descent on the feature space. This fast gradient sign method computed the gradient of the loss function with respect to the feature space, took the sign of the gradient values, and then added it to the feature values with a small constant factor. Followup work constructed more efficient attacks by iteratively applying this gradient method \cite{Kurakin2016}\cite{Dong2018} or by solving a direct constrained optimization problem \cite{Carlini2017}. 

These attacks all required access to the explicit loss function and parameter settings of the trained classifier, and so black-box models which only revealed the final class label output of the model seemed like a potential method to hide the gradients. Unfortunately, a major show-stopper with black-box models is the phenomenon of transferability \cite{Papernot2016}, where an adversarial perturbation computed for an independently trained model has a high chance of being a successful attack against a separate black-box oracle model. This independently trained model is called a substitute model. Even if the adversary is only given black-box oracle access to predicted labels, existing machine learning models are vulnerable to transfer learning attacks executed by training substitute models \cite{Papernot2016b}. The transfer success rate is the probability that an adversarial example computed for the substitute model is also misclassified by the black-box oracle.

Direct query-based attacks such as zeroth order optimization \cite{Chen2017a} and boundary attack \cite{Brendel2018} have also emerged as alternative black-box attacks without training substitute models. These attacks initialize with any misclassified data point on the other side of the decision boundary and iteratively perform rejection sampling to find a misclassified point closer to the decision boundary. This technique requires at least $10^4$ adaptive queries to the classifier, which means the choice of the next query point depends on the result obtained for the previous query points. In contrast, transfer-based attacks from training substitute models can succeed using a much smaller number of between $0$ to $10$ epochs of adaptive queries, where multiple queries can be presented simultaneously in each epoch.

Researchers have tried many avenues of constructing defenses to prevent these attacks. Previous work has attempted to train models to be explicitly robust to attacks by incorporating robustness into the optimization problem \cite{Madry2017}\cite{Schott2018}, by input transformations and discretization to reduce model linearity \cite{JacobBuckmanAurkoRoyColinRaffell2018}, or by injecting randomness at inference time \cite{Xie2017}. However defenses based on robust training have been subsequently broken by changing the space of allowable perturbations \cite{Sharma2017}, and other defenses have been broken by more sophisticated attacks \cite{Athalye2018}. 

Recent explanations suggest that the existence of adversarial examples is actually inevitable in high-dimensional spaces. \cite{Goodfellow2015} \cite{Gilmer2018}\cite{Ford2019} suggest that these examples exist for any linear classifier with nonzero error rate under additive Gaussian noise. This vulnerability is a simple geometrical fact when the dimension $d$ is large: because most of the mass of a Gaussian distribution is concentrated near the shell, the distance to the closest misclassified example is a factor $d^{1/2}$ closer than the distance to the shell. \cite{Ilyas2019} argue that adversarial perturbations can actually be robust features for generalization, and thus their adversarial nature is just a misalignment with our natural human notions of robustness. 

In light of the evidence for the inevitability of adversarial perturbations, one goal we can still hope to achieve is a computational separation between declaring their existence and finding one. We propose a solution which uses hidden random bits that behave like a cryptographic key, meaning that any instantiation of the random bits works with high probability, but an attacker should not be able to attack the overall classifier without knowing the random bits. The space of all possible random bits in our construction will be exponential in the number of classes, so guessing the random bits is intractable. 

In order to hide the randomness in a single classifier, we use a black-box ensemble scheme in which the adversary learns only the output of the overall ensemble without learning the output of any individual classifiers. Previous ensemble techniques for increasing adversarial robustness only subsample or augment the training data within each class \cite{Tramer2017}, whereas our ensemble samples random splits of the labels themselves within the overall multiclass classification setup. This means that the underlying classification problem is unknown to the adversary, and we argue that this randomness decreases the transfer success rate. In addition, our ensemble construction is allowed to abstain from making a prediction, which behaves functionally like a built-in adversarial example detector and amplifies the robustness gain within each individual classifier.

Because the scope of attacks an adversary can mount is so large, we restrict our adversary to a constant number of epochs of adaptive queries. This still captures practical attacks such as transfer-based attacks that train substitute models from a constant number of epochs, but does not capture iterative attacks making tens of thousands of adaptive queries. In the case of just a single epoch of adaptive queries, we prove that the adversarial test error converges to twice the standard test error as the number of classes increases. The proof is based on a new security assumption which is in principle simpler to empirically verify than the entire construction, and we provide evidence for it on CIFAR-10 against projected gradient descent \cite{Kurakin2016} and momentum iterative gradient method \cite{Dong2018} attacks. We also provide empirical evidence of the effectiveness of this defense against $10$ epochs of adaptive queries on the MNIST and CIFAR-10 data sets using a standard substitute model attack benchmark by \cite{Papernot2016a}. 

\section{Preliminaries}

Let $\mathcal{X}\subset\mathbb{R}^d$ be the feature space, and let $\mathcal{Y}=\{1,2,\dotsc,N\}$ be the set of classes. The learning problem is to construct a multiclass classifier $F: \mathcal{X} \rightarrow \mathcal{Y} \cup \{\omega\}$ that is allowed to abstain from making a prediction by returning the symbol $\omega$. We assume all classifier training is conducted using a fixed training algorithm for binary classification $\mathsf{ML}$ which is public knowledge. $\mathsf{ML}$ takes as input a set of binary-labeled data points $\{(x_i, z_i)\}_{i=1}^n$, where each $x_i\in\mathcal{X}$ and $z_i \in \{\pm 1\}$, and outputs a binary classifier $f: \mathcal{X}\rightarrow\{\pm 1\}$. The multiclass training data $\{(x_i, y_i\})$ is public knowledge, and the binary classifiers are trained over this data set by defining a mapping $\phi:\{1,\dotsc,N\}\rightarrow\{\pm 1\}$ that takes each data point $(x_i, y_i)$ to $(x_i, \phi(y_i))$. Furthermore, we assume that $\mathsf{ML}(\{(x_i, z_i)\})_{i=1}^n = - \mathsf{ML}(\{(x_i, -z_i)\})_{i=1}^n$, which just means that if the labels $-1$ and $1$ were reversed in the training data, then the trained classifier would be identical except for outputting the opposite sign . Lastly, we fix some space $\mathcal{P}\subset\mathcal{X}$ to be the set of allowable adversarial perturbations; a commonly used perturbation space is $\{\rho\in\mathcal{X}\,|\,\|\rho\|_{\infty}<c\}$, which for example constrains each pixel in an image to be modified by a small value. 

\subsection{Threat model}

We consider the setting of a server hosting a fixed classifier $F: \mathcal{X}\rightarrow\{1,\dotsc,N,\omega\}$ and users who interact with the server by presenting a query $q\in\mathcal{X}$ to the server and receiving the output label $F(q)$. We call $F$ a black-box classifier, because the user does not see any of the intermediate computation values of $F(q)$. Two types of users access the server: honest users who present queries drawn from a natural data distribution, and adversarial users who present adversarial examples designed to intentionally cause a misclassification. The desired property is to serve the honest users the true label while simultaneously preventing the adversarial users from causing a misclassification; the latter is accomplished by either continuing to return the true label on adversarial examples or by returning the abstain label $\omega$. 

In order for this distinction to be well-defined, we need to separate natural misclassified examples from adversarial examples. We achieve this by fixing in advance a data point $x$ which is correctly classified by $F(x)$ and requiring the adversary to compute a perturbation $\rho\in\mathcal{P}$ for this specific $x$ such that $F(x+\rho)\not\in\{F(x),\omega\}$. We think of $x$ as a parameter of the attack, for example the natural image of the face of an attacker who wishes to masquerade as someone else. The classifier $F$ is secure for $x$ if, with high probability over the construction of $F$, the adversary cannot find a $\rho$ satisfying this.

We formalize this attack problem by the notion of a \emph{security challenge}. The adversary is given all the information about $F$ except for any internal randomness used to initialize $F$. The adversary is then given the \emph{challenge point} $(x, y)$ with $F(x)=y$ being the correct classification, and the adversary successfully solves the security challenge if he finds a $\rho$ such that $F(x+\rho)\not\in \{\omega, F(x)\}$ with non-negligible probability. The solution to the security challenge is a successful attack. 

The separation between existence of a solution and feasibility of finding it is given by resource constraints on the adversary, most commonly in the form of runtime. We say that a security challenge is \emph{computationally secure} if there does not exist an algorithm for finding a solution within these resource constraints. In addition to runtime, we also consider the constraint of how many times the adversary is allowed to interact with the classifier. 

We make a distinction between these \emph{query points} (denoted by $q$) and the \emph{challenge point} (denoted by $x$), both of which are feature vectors in $\mathcal{X}$. Query points are arbitrarily chosen by the adversary for the purpose of learning more about the black-box $F$, and there is no notion of correctness for $F(q)$. The ability to obtain labels for arbitrary query points is the key factor that enables the adversary to mount more powerful black-box attacks; without query access, the attacker is limited to relatively simple transfer-based attacks from models trained on standard datasets. We leverage this distinction to obtain a provable security guarantee by using cryptographic proof techniques.

\subsection{Security proofs in cryptography}

Instead of directly trying to prove the security of $F$, we define a simpler system $f$ that is easier to empirically test and reason about. We then prove a reduction from the security challenge of $F$ to the security challenge of $f$, which shows that $F$ is at least as hard to attack as $f$. We define a \emph{security assumption} that characterizes the hardness of attacking $f$. This security assumption is not mathematically proven to be true, but nonetheless defining the right assumption makes the reduction is useful, because this assumption can be easier to empirically study. If the security assumption is true, then $F$ is secure. The security assumption we define is the hardness of attacking a new type of randomized classifier without any query access to it.

\subsection{Random binary classifiers}

In a multiclass classification problem with labels $1,\dotsc,N$, suppose we have a binary classifier $f: \mathcal{X}\rightarrow\{\pm 1\}$ for two particular classes $y$ and $t$, where class $y$ is mapped to $+1$ and class $t$ is mapped to $-1$. An adversary is given a data point $(x,y)$ with $f(x)=+1$, and the adversary wishes to attack this binary classifier by computing a perturbation $\rho$ such that $f(x+\rho) = -1$. If $f$ were a standard binary classifier trained on the $y$ versus $t$ classification problem, then this would be a straightforward transfer attack scenario. However, instead $f$ is trained with all remaining $N-2$ classes also having been randomly remapped to $\pm 1$ with equal probability. In other words, for each class $k\not\in\{ y,t\}$, we sample a Rademacher random variable $z_k \sim\{\pm 1\}$ and assign every data point of original label $k$ to the new binary label $z_k$. This random assignment does not change the original $y$-vs-$t$ classification task when all data points are only of original class $y$ or $t$. The resulting $f$ corresponding to training with the random binary labels $\{\pm 1\}^{N-2}$ is a \emph{random binary classifier}:

\begin{definition}[Random binary classifier] Let $\mathcal{D}$ be a distribution over $\{\pm 1\}^N$. The \emph{random binary classifier} over $\mathcal{D}$ is the distribution of $f$ over $z\sim\mathcal{D}$ where each training data point $x_i$ is relabeled to $\pm 1$ by $z_{y_i}$:
	\begin{align*}
	f_z := \mathsf{ML}\left(\left\{(x_i, z_{y_i})\right\}_{i=1}^n\right).\qed
	\end{align*}
\end{definition}


The security challenge for the random binary classifier is to compute a perturbation that changes its output with high probability over the sampling of $z$. 

\begin{definition}[Security challenge for random binary classifier]\label{def:random_classifier_challenge}
	Let $f_z := \mathsf{ML}\{(x_i, z_{y_i})\}_{i=1}^n$. Let $z\sim\{\pm 1\}^N$ be a Rademacher random vector, and let $\mathcal{D}_{yt}$ be the distribution of $z$ conditioned on $z_y = +1, z_t = -1$. The \emph{security challenge} for a challenge data point $(x, y)$, failure rate $\delta>0$, and target label $t\neq y$ is to compute a perturbation $\rho \in \mathcal{P}$ which changes the output of $f_z(x)$ with failure rate no greater than $\delta$:
	\begin{align*}
	\mathop{\Pr}_{z\sim\mathcal{D}_{yt}}[f_z(x+\rho) \neq f_z(x)] > 1-\delta.
	\end{align*}
	In particular, the adversary has no ability to obtain labels for query points from the random binary classifier. \qed
\end{definition}

Note that the adversary has knowledge of two of the bits of $z$, corresponding to the original label $y$ and some target label $t\neq y$. Our security assumption is that for any $\rho\in\mathcal{P}$, there is enough randomness in the remaining $N-2$ data classes such that the failure rate is non-negligible.

\begin{assumption}[Security assumption]\label{def:security_assumption}
	Given an instance of the security challenge for a random binary classifier with parameters defined as in Definition \ref{def:random_classifier_challenge}, for any $\rho\in\mathcal{P}$, for all $c > 0$, there exists a constant $N_0 > 0$ such that
	\begin{align*}
		\mathop{\Pr}_{z\sim\mathcal{D}_{yt}} \left[ f_z(x+\rho) \neq f(x)\right] \le 1 - 1/N^c
	\end{align*}
	whenever $N \ge N_0$. \qed
\end{assumption}

Note that this implicitly assumes $\mathcal{P}$ does not contain any non-adversarial perturbations, such as those of the form $x'-x$ where $x'$ is a legitimate image of class $t$. This assumption also does not place any computational constraints on the adversary yet; the security comes from the randomness in $z\sim\mathcal{D}_{yt}$, which is sampled after $\rho$ is already fixed. In Section \ref{sec:perturb}, we experimentally justify this assumption by estimating the transfer success probability for all pairs of classes $(y,t)$ in the CIFAR-10 dataset using the standard $\ell_{\infty}$-ball for $\mathcal{P}$ and two different state-of-the-art transfer attacks.

We give two reasons why this assumption is the right one to make. Firstly, the scope of attacks to analyze is greatly reduced when the attacker has no access to the classifier. The adversary can essentially only mount transfer learning attacks by training models on the public dataset. Secondly, we only require the probability of success of the adversary to be bounded below $1$ by a constant, and the overall security of the ensemble can be boosted from this bound. 

\subsection{Main construction}

Recall that our goal is to construct a multiclass classifier $F: \mathcal{X} \longrightarrow \{1,2,\dotsc,N, \omega\}$ which is allowed to abstain from making a prediction (as represented by the output $\omega$), and an adversarial perturbation $\rho$ is only considered a successful attack if $F(x+\rho) \not\in \{F(x), \omega\}$. 

Our ensemble construction is the error-correcting code approach for multiclass-to-binary reduction \cite{Dietterich1994}, except with completely random codes for security purposes.

\begin{construction}[Random ensemble classifier]\label{construction:ensemble}
	Given a multiclass classification problem with labels $\mathcal{Y}=\{1,\dotsc,N\}$, a codelength $M$, and a threshold parameter $r \in (0, 1/3)$:
	\begin{itemize}
		\item Sample random matrix $Z \in \{\pm 1\}^{N\times M}$, where each $Z_{ij}\sim \{\pm 1\}$ independently and with equal probability
		\item For $j=1,\dotsc,M$, construct the binary classifier $f_j = \mathsf{ML}\left(\{(x_i, Z_{y_ij})\}_{i=1}^n\right)$
	\end{itemize}
	Given a query data point $x$, compute output $F(x)$ by:
	\begin{itemize}
		\item Compute the predicted codeword vector $C(x) := (f_1(x),\dotsc, f_M(x))$
		\item Compute $(d^*, y^*) = \displaystyle\mathop{\min}_y \|Z_y - C(x)\|_H$, where $y^*$ is the index and $d^*$ is the Hamming distance to $Z_{y^*}$
		\item If $d^* < Mr$, then output $y^*$, else output $\omega$\qed
	\end{itemize}
\end{construction}

In this construction, the codeword $Z_y \in \{\pm 1\}^M$ acts as the identity of class $y$, and thus the classification of a data point $x$ is the class codeword which is closest to its predicted codeword $C(x)$. We should think of the free parameters as $M = \Omega(\text{poly}(N))$ and $r = O(1/N)$. $M$ needs to be sufficiently large in order for the random ensemble classifier to be accurate on natural examples. The parameter $r$ should be greater than the standard test error of a trained classifier, or otherwise the ensemble will abstain on too many legitimate test samples. However $r$ must be small enough for security purposes, which we will quantify in our main theorem.

We give some intuition for why this construction has desirable security properties. In order for an adversary to change the overall output of some test point $(x,y)$, he needs to change the output of sufficiently many binary classifiers $f_j$ so that $C(x+\rho)$ is close to some codeword $Z_t, t\neq y$. But the Hamming distance between $Z_y$ and $Z_t$ is $M/2$ on expectation, and $x, x+\rho$ must be within distance $Mr$ to $Z_y, Z_t$ respectively. Since each $f_i$ is constructed independently at random, the overall probability of success is exponentially decreasing in the probability of successfully changing the output of an individual classifier.

We proceed to define the security challenge for this construction. We will use the shorthand notation $Z\sim\{\pm 1\}^{N\times M}$ to denote the distribution of $Z\in\{\pm 1\}^{N\times M}$ where each entry is independently sampled from $\{\pm 1\}$ with equal probability. 

\begin{definition}[Security challenge for random ensemble]\label{def:random_ensemble_challenge}
	Let $F_Z(\cdot)$ be the ensemble classifier constructed with random hidden code matrix $Z$ as defined in Construction \ref{construction:ensemble}. The \emph{security challenge} for a challenge data point $(x, y)$ and accuracy $\epsilon\in(0,1)$ is a two-round protocol:
	
	\begin{enumerate}
		\item Provide $Q$ nonadaptive queries to $F_Z(\cdot)$ and receive answer labels, denoted by $\{(q_k, a_k)\}_{k=1}^Q$. The queries cannot depend on the hidden random code $Z$, but can otherwise depend on the public information such as the training data and the oracle $\mathsf{ML}$.
		\item Return a perturbation $\rho\in\mathcal{P}$ by some function of the query answers $\rho = \phi(\{a_k\}_{k=1}^Q)$ such that $\rho$ satisfies
	
	\[\mathop{\Pr}_{Z\sim\{\pm 1\}^{N\times M}} \left[F_Z(x+\rho)\not\in \{F_Z(x), \omega\} \right] > \epsilon,\]
	
	\end{enumerate}
An algorithm for solving the security challenge is determined by its query set $\{q_k\}_{k=1}^Q$ and the function $\phi$ for computing the final perturbation from the query answers. \qed
\end{definition}

For example, one possible attack captured by this definition is training a substitute model with a one epoch of data augmentation obtained from querying the classifier, as described by \cite{Papernot2016b}. The adversary starts with a pre-labeled dataset of arbitrary size, usually the public training data set, and trains an initial substitute model. The adversary then refines this initial model by using Jacobian data augmentation to add new synthetic data points to the training data. In each epoch of data augmentation, the adversary obtains labels for these synthetic points using the black-box classifier.

The synthetic data points are the queries $q_1,\dotsc, q_Q$, and thus our proof guarantees security against a single epoch of data augmentation. The actual implementation of this attack in \cite{Papernot2016a} uses $10$ substitute training epochs, and our proof does not apply directly to this implementation, because the second round of queries can depend on the answers in the first round. Nonetheless, we show empirically in Section \ref{sec:empirical_substitute} that our construction is still secure against the benchmark of $10$ data augmentation epochs.

\section{Security results}\label{sec:theory}

The main theoretical result is a reduction from solving the random classifier challenge to solving the random ensemble challenge. In our reduction, we make the simplifying assumption that the space of allowable perturbations $\mathcal{P}$ is the same in both security challenges. This allows us to get away with not explicitly defining which perturbations are adversarial and which are legitimate, because a perturbation which makes $x+\rho$ a legitimate image of the class $t$ would solve both security challenges simultaneously. We also assume without loss of generality that $r$ is chosen such that $Mr \in \mathbb{Z}$, because Hamming distance is an integer. 


\begin{theorem}\label{thm:main}
	Suppose there exists an algorithm $\mathcal{A}$ that can solve the security challenge for the random ensemble with any threshold $r\in(0,1/2)$ such that $Mr\in\mathbb{Z}$ using $Q$ queries and with accuracy $\epsilon \in(0,1)$. Then there is an algorithm that can compute a perturbation $\rho$ which solves the security challenge for a random binary classifier with failure rate
	
	\[\delta < 2\left(r + \sqrt{\frac{\log(1/\epsilon)}{2M}}\right).\]
	
	The algorithm succeeds in computing this perturbation with probability (over $Z$) at least
	
	\[1 - 4NQ\sqrt{\frac{1-r}{2\pi Mr}} 2^{-M(1-H_2(r))},\]
	
 	where $H_2(r) = -r\log_2 r - (1-r)\log_2 (1-r)$ is the negative entropy function and can be bounded away from $1$ when $r$ is bounded away from $1/2$. 

\end{theorem}

The theorem shows that if such an algorithm $\mathcal{A}$ exists, $r = O(1/N)$, and $M=\Omega(\text{poly}(N))$, then the failure rate decreases as $O(1/N^c)$ for some constant $c$, which contradicts the security assumption (Assumption \ref{def:security_assumption}). Conversely, if the security assumption is true, then an adversary cannot solve the security challenge for the random ensemble with $O(\text{poly}(N))$ nonadaptive queries to the ensemble classifier. When $r < \frac{\delta}{2}$ and the security assumption is true, the theorem gives the following upper bound on the adversarial test error:

\begin{align*}
	\epsilon &< \exp\left(-2M \left(\frac{\delta}{2} - r\right)^2 \right).
\end{align*}

Recall that the parameter $r$ needs to be greater than the standard test error of a random binary classifier for good standard test accuracy of the ensemble, but less than $\frac{\delta}{2}$ for good adversarial accuracy. The more accurate each random binary classifier is, the smaller we can set the value of $r$ to be, which in turn gives a smaller upper bound on the adversarial test error $\epsilon$. This shows that our definition of adversarial test error is compatible with standard test error. 

We give a brief proof sketch here, deferring the full proof to Section \ref{sec:proofs}. Given a single random classifier $f_z$, we can simulate the entire ensemble classifier $F_Z$ by constructing the remaining $M-1$ random classifiers using the public data set and $\mathsf{ML}$. However, we cannot apply $\mathcal{A}$ to $F_Z$ directly, because in Definition \ref{def:random_classifier_challenge} there is no query access to $f_z$. Thus we first show in Lemma \ref{lem:trim} that we can simulate the output of the entire ensemble using only $M-1$ classifiers with high probability.

Applying the algorithm $\mathcal{A}$ the ensemble of $M-1$ classifiers produces an attack perturbation $\rho$. Since this simulates the ensemble of $M$ classifiers with high probability, then this attack perturbation also applies to the entire ensemble of $M$ classifiers. Now we want to compute the probability of the output of each individual classifier in the ensemble being changed, but the $Q$ queries could potentially leak information about some column $Z^j$. We use Lemma \ref{lem:trim} for each column $j$ to show that this is not the case; i.e. that the query answers are completely determined by the remaining $M-1$ columns with high probability and thus independent of column $j$ itself. Then we show in Lemma \ref{lem:binomial} that an overall success probability of $\epsilon$ gives an upper bound on $\delta$ for each individual classifier.

\section{Empirical results}

We provide empirical analysis on the security assumption (Assumption \ref{def:security_assumption}) and the adversarial test accuracy against black-box substitute model training attacks for the MNIST \cite{LeCun1998} and CIFAR-10 \cite{Krizhevsky2009} datasets. We use code from the CleverHans adversarial examples library \cite{Papernot2016a} and from the MadryLab CIFAR10 adversarial examples challenge \cite{Madry} for the base classifier architecture, training, and attacks. The only modification to the base classifier architecture was to change the output layer from dimension $10$ to dimension $2$ for a binary output; no further architecture tuning was performed to optimize natural accuracy. 

\subsection{Analysis of random binary classifiers}\label{sec:perturb}

First, we empirically estimate the transfer success rate for all pairs of classes. We train a sample size of 40 random binary classifiers and then compute an adversarial perturbation for each test data point and each target class. The perturbation is computed by using a pre-trained standard model for the respective dataset with all $N$ output dimensions. We then compute whether each random binary classifier makes a different prediction on the original test data point versus the perturbed test data point. Finally, for each pair $(y,t)$, we empirically estimate the probability of the output of $f_z(\cdot)$ being changed conditioned on $z_y \neq z_t$ and plot this. The goal of this analysis is to show that this probability is bounded below $1$ by a constant.

We use the Projected Gradient Descent and the Momentum Iterated Gradient Descent transfer attacks on the cross-entropy loss with an $\ell_{\infty}$ norm bound of $\epsilon = 8$. The pre-trained substitute is a w28-10 wide residual network \cite{Zagoruyko2016}, and the random binary classifiers are the same ResNet architecture but with two output dimensions instead of ten. We visualize the average-case success probability in an $N\times N$ grid where the $(y,t)$ coordinate shows the attack success probability over original data points of class $y$ and target label $t$. The color of each cell represents the probability using the Viridis color palette shown in Figure \ref{fig:viridis}.

\begin{figure}[ht]
	\centering
	\subfloat{{\includegraphics[width=0.45\columnwidth]{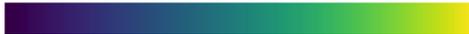}}}
	\caption{Viridis color palette, uniformly scaled from $0$ to $1$}
	\label{fig:viridis}	
\end{figure}

Figure \ref{fig:cifar10_confusion} shows the empirical success probabilities of the attack over the CIFAR-10 data set for all pairs of classes. 

\begin{figure}[ht]
	\centering
	\subfloat{{
	\includegraphics[width=0.45\columnwidth]{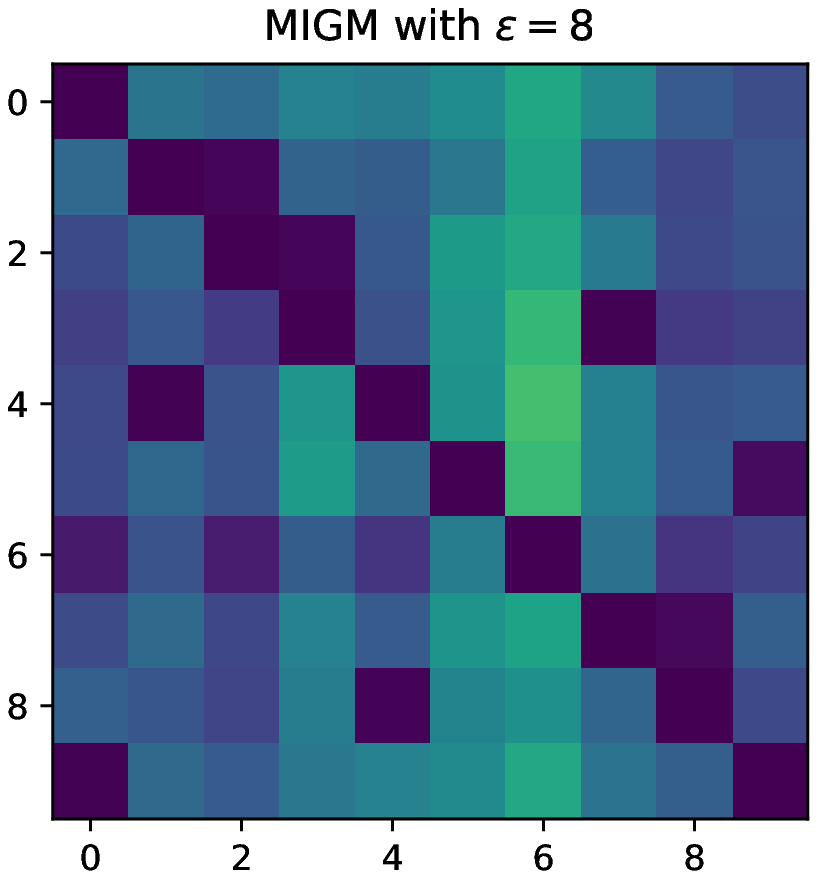}}}
	\qquad
	\subfloat{{\includegraphics[width=0.45\columnwidth]{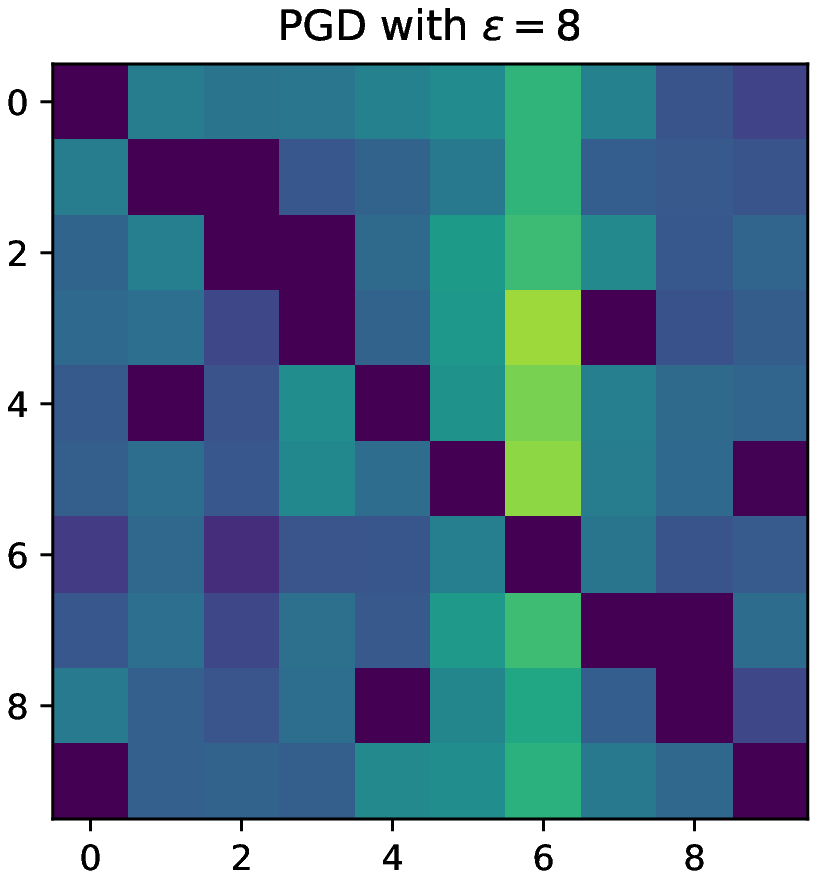}}}
	\caption{Success probabilities for targeted attacks on CIFAR-10 random binary classifiers}
	\label{fig:cifar10_confusion}	
\end{figure}

In the image, the cell $(4,6)$ appears to have the highest probability, and the entire column $y=6$ (frog) appears to have particularly high average success rate as a target class. For our security definition, we are interested in worst-case attack success rates, so we plot the distribution over each test data point for the $(y,t)$ pairs $(4,6)$ and $(5,3)$. Figure \ref{fig:cifar10_pgd_dist} and Figure \ref{fig:cifar10_migm_dist} show the individual success rates for MIGM and PGD, respectively.

\begin{figure}[ht]
	\centering
	\subfloat{{\includegraphics[width=0.45\columnwidth]{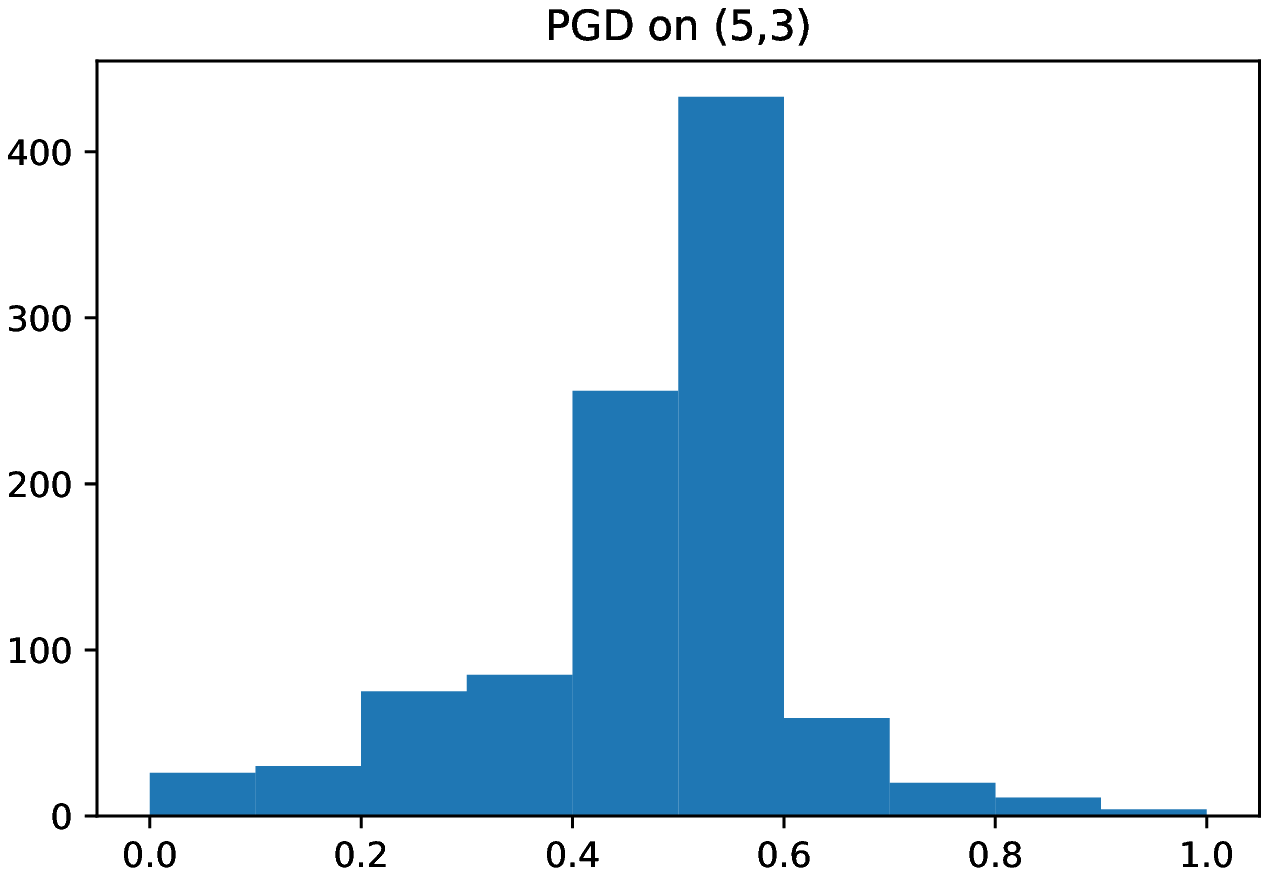}}}
	\qquad
	\subfloat{{\includegraphics[width=0.45\columnwidth]{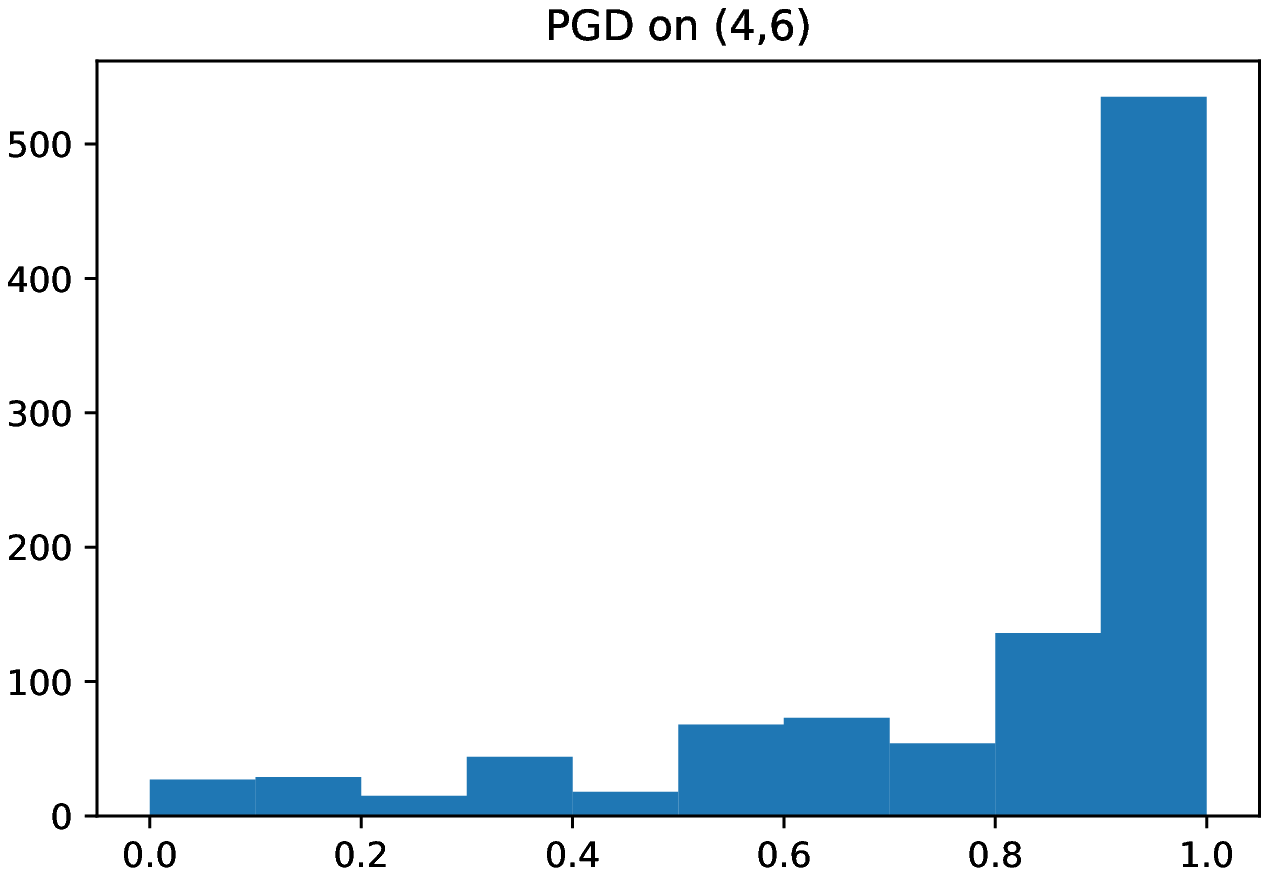}}}
	\caption{Distribution of success probabilities for individual CIFAR10 test data points under PGD attack}
	\label{fig:cifar10_pgd_dist}
\end{figure}

\begin{figure}[ht]
	\centering
	\subfloat{{\includegraphics[width=0.45\columnwidth]{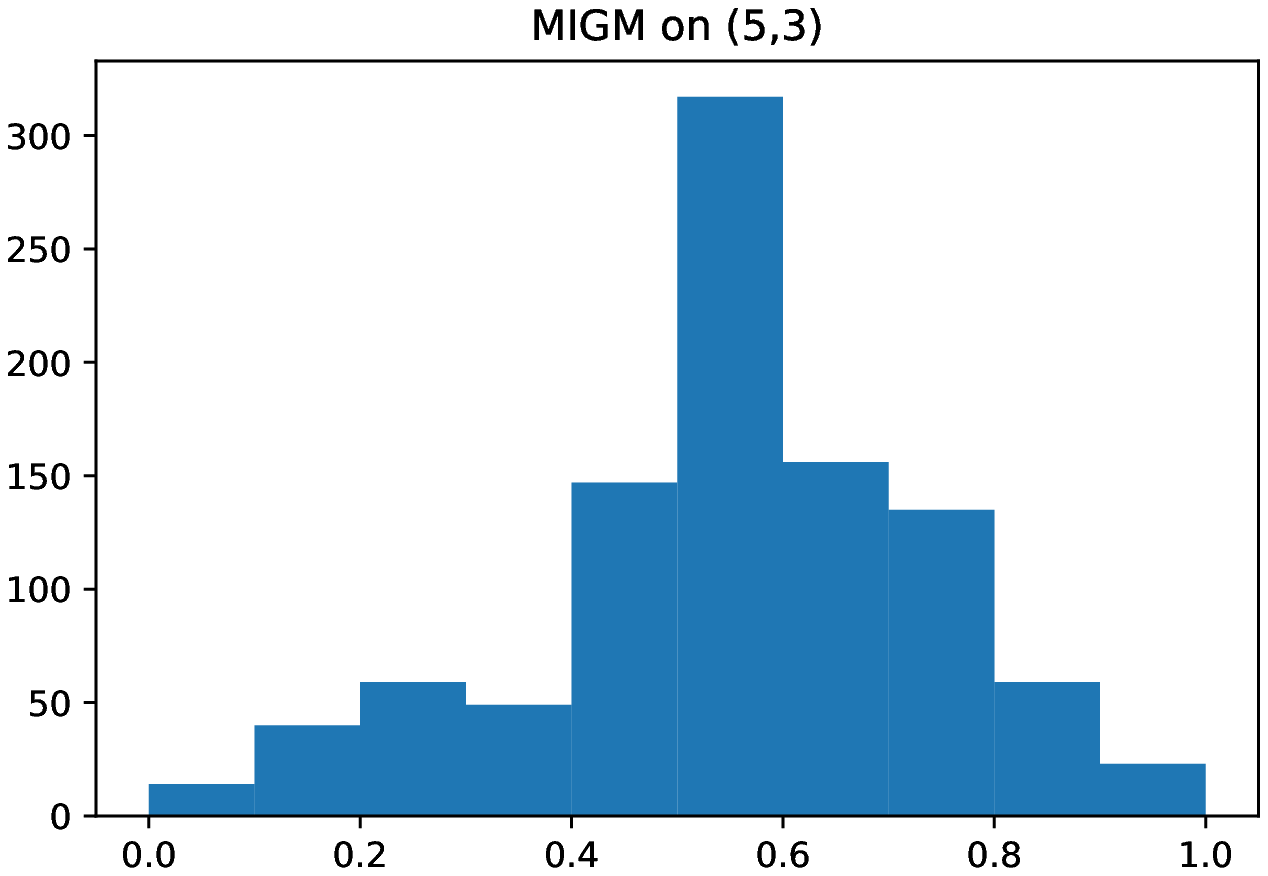}}}
	\qquad
	\subfloat{{\includegraphics[width=0.45\columnwidth]{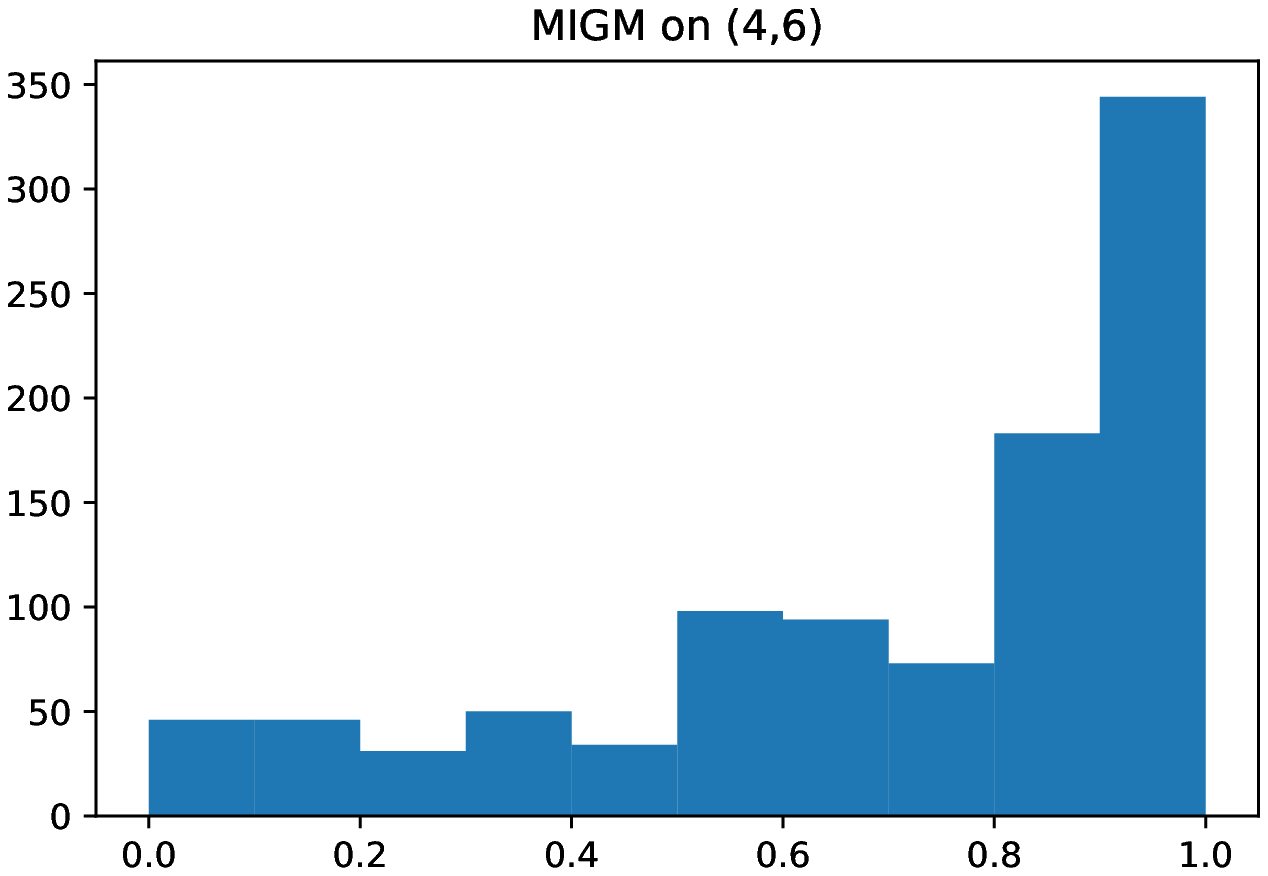}}}
	\caption{Distribution of success probabilities for individual CIFAR10 test data points under MIGM attack}
	\label{fig:cifar10_migm_dist}
\end{figure}

We see that among the $(y,t)$ pairs where $t\neq 6$, the  security definition needed for our main theorem is satisfied with high probability over the test examples. However, many of the examples are vulnerable to a targeted attack with target class $t=6$. This suggests that the Frog class is especially distinct from the other $9$ classes, such that even when it is randomly included in a binary partition, the neural network still builds a kind of frog detector separate from the other randomly included classes. 

\subsection{Analysis of black-box adversarial accuracy}\label{sec:empirical_substitute}

Next, we empirically analyze the robustness of our random ensemble construction to black-box transfer learning attacks. Instead of performing a transfer attack from a standard model, these attacks train a specific substitute model by querying the black-box classifier directly. We use the CleverHans attack library \cite{Papernot2016a} to benchmark this. The attack algorithm trains a two-layer fully connected substitute model iteratively augmenting its training data set via queries to the random ensemble scheme and then uses the Fast Gradient Sign Method on the substitute model. 

Because the attack library is not designed for querying classifier which abstains, we perform substitute model training with a non-abstaining random ensemble (i.e. $r = 1/2$). We consider the threshold $r$ at the end when analyzing the final true and adversarial test accuracies. In order to incorporate the abstain label, we use the following definitions of accuracy for our experiments. The true test accuracy requires the classifier to make the correct, non-abstaining prediction. However when computing adversarial accuracy, we also consider it a success if the classifier outputs $\omega$.

\begin{definition}[True and adversarial test accuracy]
	Given a multiclass classifier \\$F: \mathcal{X} \rightarrow \{1,\cdots,N,\omega\}$ which is allowed to abstain from making a prediction (as represented by the output $\omega$), the relevant accuracy benchmarks are
	\begin{align*}
		\text{True accuracy} &:= \mathop\mathbb{E}_{(x,y)} \left[\mathbbm{1}[F(x) = y]\right]\\
		\text{Adversarial accuracy} &:= \mathop\mathbb{E}_{(\hat{x},y)} \left[\mathbbm{1}[F(\hat{x}) \in \{y, \omega\}]\right],
	\end{align*}
	
	where $x$ is the original data point and $\hat{x}$ is an adversarial perturbation of $x$. \qed
\end{definition}

All random binary classifiers used in these experiments are the same architecture as the random binary classifiers in Section \ref{sec:perturb}. Figure \ref{fig:accuracy_vs_robustness} shows that the ensemble enjoys good adversarial accuracy in the low-$r$ regime.

\begin{figure}[ht]
	\centering
	\subfloat{{\includegraphics[width=0.47\columnwidth]{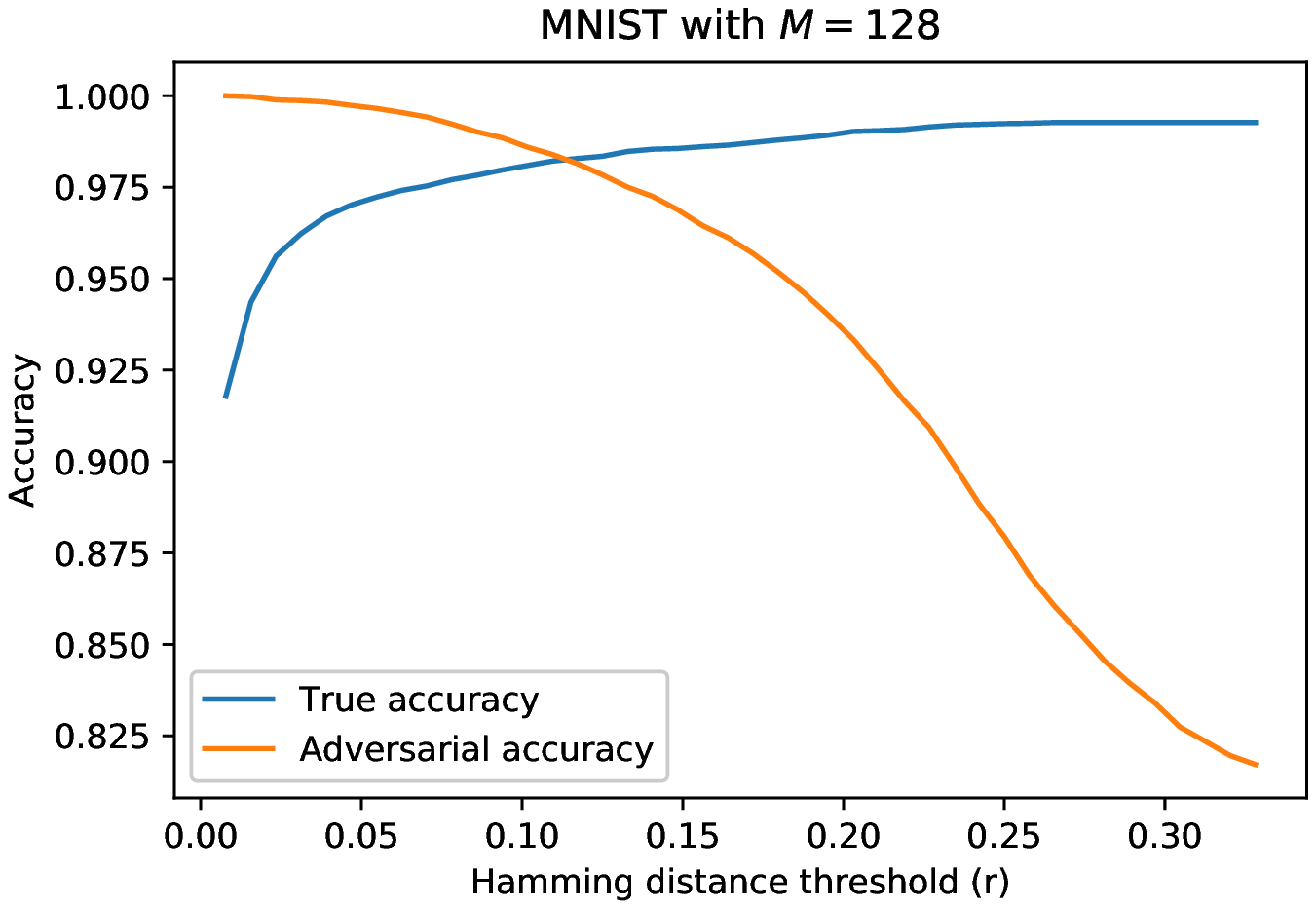}}}
	\qquad
	\subfloat{{\includegraphics[width=0.47\columnwidth]{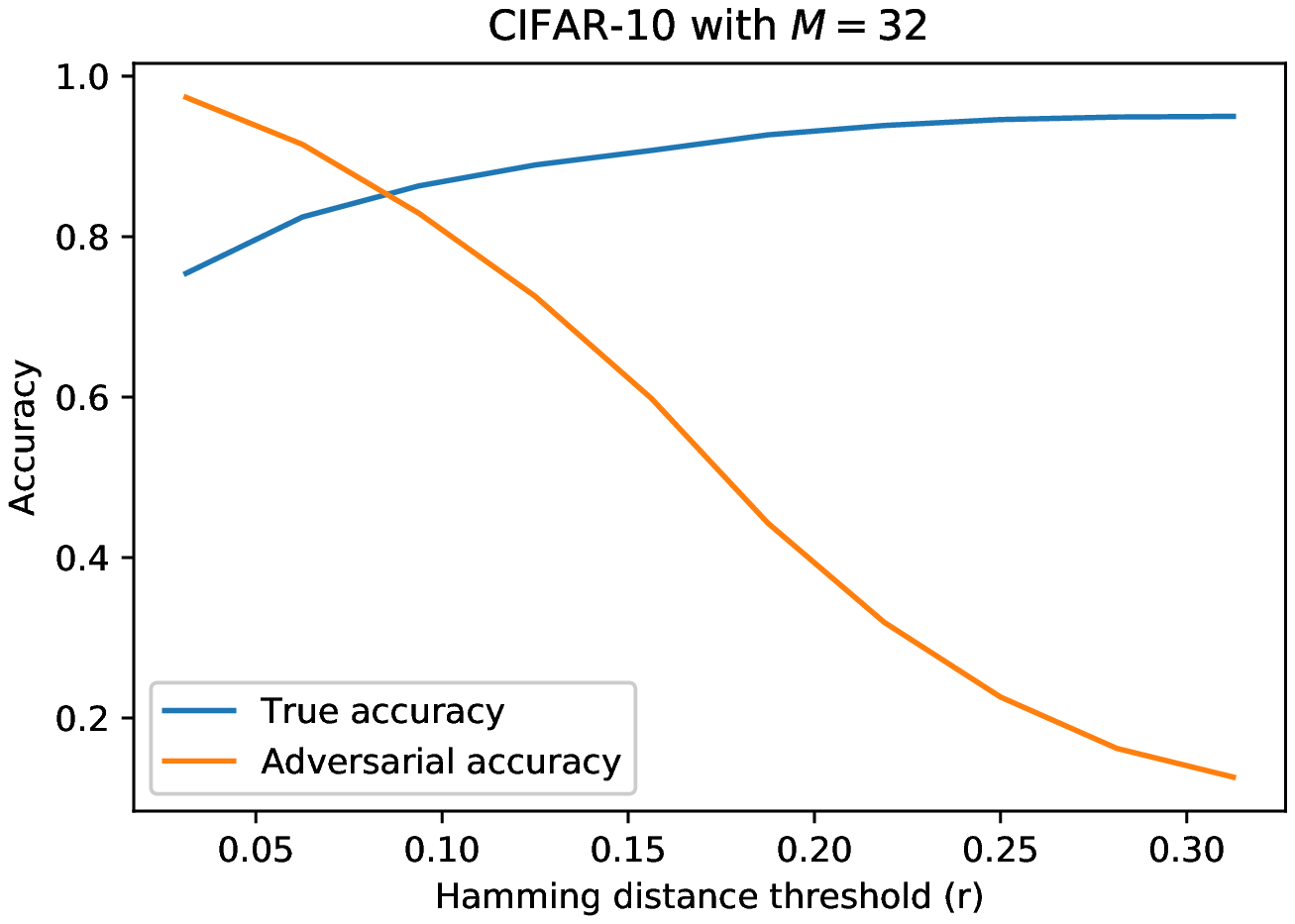}}}
	\caption{Accuracy versus Hamming distance ratio ($r$)}
	\label{fig:accuracy_vs_robustness}
\end{figure}

\section{Discussion}

The experiments demonstrate that random ensembles are a promising approach to security against black-box transfer attacks. The security of our scheme can be more reliably estimated than traditional constructions. We show that the security assumption holds for the majority of pairs of classes without any modification to standard architectures used on CIFAR-10. Security properties of our construction will improve for larger classification problems with more classes, although the number of random classifiers needed will as well. 

One important item to note is that the random ensemble construction is not compatible with standard techniques of robust training. Robust training tends to decrease the standard test error of the classifier, which means that a larger threshold $r$ needs to be used to account for natural errors in the individual random classifiers. However, a larger $r$ value leads to weaker security in the ensemble. 


\bibliography{neurips2019}
\bibliographystyle{alpha}

\appendix
\newpage
\section{Proofs}\label{sec:proofs}
\begin{lemma}\label{lem:trim}
	
	Fix any query point $q$ and threshold $r < 1/2$ such that $Mr\in\mathbb{Z}$. Given a random ensemble function $F_Z: \mathcal{X} \rightarrow \{1,\dotsc,N\}$ with $M$ independently and identically generated random classifiers and threshold $r < 1/3$, fix some $j\in\{1,\dotsc,M\}$ and let $F_{Z^{-j}}$ denote the modified ensemble which ignores the $j$th random classifier and takes the vote over only the remaining $M-1$ classifiers. Then
	
	\begin{align*}
		\mathop{\Pr}_{Z^{-j}\sim \{\pm 1\}^{N\times(M-1)}} \left[ F_Z(q) \neq F_{Z^{-j}}(q) \right]
		&\le 4N\sqrt{\frac{1-r}{2\pi Mr}} 2^{-M(1-H_2(r))},
	\end{align*}
	where the probability is taken only over the matrix $Z^{-j}$ and is independent of the column $Z^j$. $H_2(r) = -r\log_2 r - (1-r)\log_2 (1-r)$ can be bounded away from $1$ when $r$ is bounded away from $1/2$. 
\end{lemma}

The lemma shows that for any $j$, with high probability over $Z^{-j}$ the query answer $F_Z(q)$ is independent of $Z^j$, so that no information is revealed by the queries about column $j$. In the following proofs we will use the shorthand $f_j := f_{Z^j}$, i.e. the random classifier constructed from the $j$th column of $Z$. 


\begin{proof}
	The only way the additional classification output of $f_{j}(q)$ can influence the decision of the entire ensemble of $F_{Z^{-j}}(q)$ is if the predicted codeword of length $M-1$ is on the decision boundary between some class $i$ and the abstaining space corresponding to $\omega$. In the boolean hypercube $\{\pm 1\}^{M-1}$, the number of points that are at a distance of exactly $k$ to any fixed point is $\binom{M-1}{k}$. Because we want our probability bound to hold true regardless of the value of $f_j$, we have to consider the possibility of $f_j(q)$ influencing the points on either side of the decision boundary. To account for this, we multiply the number by $2$. Then over all $N$ classes, the number of possible points on the decision boundary is at most $2N\binom{M-1}{Mr}$ by a union bound.
	
	\begin{align}
		\frac{2N}{2^{M-1}}\dbinom{M-1}{Mr}
		&= \frac{4N(1-r)}{2^{M}} \dbinom{M}{Mr}.\label{eq:hypercube_probability}
	\end{align}
	
	We now apply the binomial coefficient upper bound from \cite{MacWilliams1978}, reproduced below:
	
	\begin{lemma}\label{lem:binom_coefficient_bound}
		Suppose $\lambda n$ is an integer, where $0 < \lambda < 1$. Then
		\begin{align*}
		\dbinom{n}{\lambda n} &\le \frac{1}{\sqrt{2\pi n\lambda (1-\lambda)}}2^{n H_2(\lambda)},
		\end{align*}
		where $H_2(\lambda) = -\lambda \log_2\lambda - (1-\lambda) \log_2(1-\lambda)$ is the negative entropy function.
	\end{lemma}

	This gives the result
	
	\begin{align*}
		\dbinom{M}{Mr} &\le \frac{1}{\sqrt{2\pi M r(1-r)}} 2^{MH_2(r)},
	\end{align*}

	where $H_2(r) = -r\log_2r-(1-r)\log_2(1-r)$ is the negative entropy function. Thus the probability in (\ref{eq:hypercube_probability}) can be bounded by
	
	\begin{align*}
		\frac{4N(1-r)}{2^{M}} \frac{1}{\sqrt{2\pi Mr(1-r)}} 2^{M H_2(r)} 
		&\le 4N\sqrt{\frac{1-r}{2\pi Mr}} 2^{-M(1-H_2(r))}.
	\end{align*}
	
\end{proof}

	Since $H_2(r)$ is bounded away from $1$ when $r$ is bounded away from $1/2$, this gives an exponentially decaying probability bound in $M$. 
	
The next lemma is a concentration result that holds when no information is revealed by the queries about any individual column. 

\begin{lemma}\label{lem:binomial}
	Suppose that the event $f_j(x+\rho)\neq f_j(x)$ is independent and identical for each column $j$. Fix a data point $(x,y)$. Given a perturbation $\rho$ which solves the security challenge for the random ensemble with target probability $\epsilon > 0$, then for every random classifier in the ensemble, $\rho$ solves the security challenge for it with failure rate $\delta < 2(r + \sqrt{\log(1/\epsilon) / 2M})$
\end{lemma}

\begin{proof}
	
Recall that the adversary is said to have solved the security challenge for the random ensemble if the vector of code bits $C_Z(x+\rho) := (f_1(x+\rho),\dotsc,f_M(x+\rho))$ has Hamming distance less than $Mr$ to any other codeword $Z_i$, where $i\neq y$. Since each entry of the code matrix is sampled independently, we can consider the probability of this event bit-by-bit.

Let $\mathcal{E}_{tj}$ be the event where $f_j(x+\rho) = Z_{tj}$. Let $\mathcal{E}_t$ be the probability of the event where $\|C_Z(x+\rho') - Z_t\|_1 \le Mr$, meaning the codeword for class $t$ is the closest. By the independence assumption, we have $\Pr[\mathcal{E}_t] = \Pr[X > M(1-r)]$ where $X \sim \text{Binom}(M, \Pr[\mathcal{E}_{tj}])$, or equivalently,

\begin{align}
	\Pr[\mathcal{E}_t] &= \Pr \left[ X < Mr \, | \, X\sim\text{Binom}(M, 1-\Pr[\mathcal{E}_{tj}])\right]. \label{eq:one_column}
\end{align}

The probability of changing $F(x)$ from $y$ to any other class can be bounded by applying the union bound to all $t\neq y$. We obtain

\begin{align*}
	\Pr[F_Z(x+\rho) \neq F_Z(x)] &\le (N-1) \Pr[\mathcal{E}_t],
\end{align*}
and by the assumption of the lemma we know the left-hand side probability is $\delta > 0$. Thus we just need to compute $\Pr[\mathcal{E}_{ij}]$ and apply a tail inequality for the binomial distribution. 

Fix one underlying code bit $j$ and some other class $t \neq y$. Each bit $Z_{tj}$ differs from the corresponding bit of $C_{yj}$ with probability $1/2$ under the random code sampling scheme. Without loss of generality, we'll let $Z_{yj} = +1$. We analyze the probability of the event $f_j(x + \rho) = Z_{tj}$ by conditioning on $Z_{tj}$, obtaining

\begin{align*}
	\Pr\left[\mathcal{E}_{tj}\right] &= 
	\Pr\left[Z_{tj} = -1\right]\Pr\left[f_j(x+\rho)=-1|Z_{tj}=-1, Z_{yj}=1\right]\\
	&+ \Pr\left[Z_{tj}=+1\right]\Pr\left[f_j(x+\rho)=+1|Z_{tj}=+1, Z_{yj}=+1\right].
\end{align*}

We note that the term $\Pr[f_j(x+\rho)=-1|Z_{tj}=-1, Z_{yj}=+1]$ is exactly the the probability $1-\delta$ in Definition \ref{def:random_classifier_challenge}. Then $\text{Pr}[\mathcal{E}_{tj}]$ can be bounded by

\begin{align*}
	\Pr[\mathcal{E}_{tj}] &\le \frac{1}{2} (1-\delta) + \frac{1}{2} (1) = 1 - \frac{\delta}{2}.
\end{align*}

Then the probability in (\ref{eq:one_column}) can be bounded by using Hoeffding's inequality, which states that given $X \sim \text{Binom}(M,p)$, for any $\alpha > 0$,

\begin{align*}
\Pr\left[X \le (p-\alpha)M \right] &\le \exp\left(-2M\alpha^2\right). 
\end{align*}

We let $X = \sum_j \mathcal{E}_{tj}$, so $p < 1 - \frac{\delta}{2}$ and $\alpha < p + r - 1 = r - \frac{\delta}{2}$. Applying Hoeffding's inequality with these parameters yields

\begin{align*}
	\Pr[\mathcal{E}_t \le Mr] &\le \exp\left( -2M\left(r-\frac{\delta}{2}\right)^2\right).
\end{align*}

$\Pr[\mathcal{E}_t \le Mr]$ is the probability of the perturbation $\rho$ solving the security challenge for the random ensemble, so by the assumption in the lemma, this is at least $\epsilon$. Thus we obtain

\begin{align*}
	\epsilon &\le \exp\left( -2M\left(r-\frac{\delta}{2}\right)^2\right).\\
	\intertext{We solve for $\delta$ as a function of $\epsilon$ to obtain the failure probability of solving the security challenge for an individual classifier:}
	\log(1/\epsilon) &\ge 2M\left(\frac{\delta}{2} - r\right)^2\\
	\delta &\le 2\left(r + \sqrt{\frac{\log(1/\epsilon)}{2M}}\right).
\end{align*}

\end{proof}

\begin{proof}[Proof of Theorem \ref{thm:main}]

We are given an instance of the security challenge for a random binary classifier (Definition \ref{def:random_classifier_challenge}). Let $f_{\overline{z}}$ be the random binary classifier, where $\overline{z} \sim \{\pm 1\}^N$ is uniformly sampled. We can simulate an entire random ensemble by constructing $M-1$ additional random classifiers in the same way that $f_{\overline{z}}$ is sampled, so that $f_1 = f_{\overline{z}}$ and $f_2,\dotsc,f_M$ are freshly sampled. Let $Z^{-j}$ denote the matrix $Z$ without the $j$th column, so that $F_{Z^{-j}}: \mathcal{X}\rightarrow \{1, \dotsc, N\}$ denotes the output of the random ensemble ignoring $f_j$. 

By the definition of the security challenge, the adversary cannot query $f_1$; however since $F_{Z^{-1}}$ is simulated by the adversary, he can make queries to $F_{Z^{-1}}$ and run $\mathcal{A}$ to produce a perturbation $\rho$ attacking $F_{Z^{-1}}$. But if $F_{Z^{-1}}(q_i) = F_Z(q_i)$ for each query $q_i$, then $\mathcal{A}$ would have produced the same perturbation $\rho$ attacking $F_Z$. 

By Lemma \ref{lem:trim} and a union bound over the number of queries, the hypothetical query answers $a_1,\dotsc,a_Q$ to the entire ensemble $F_Z$ depend only on $F_{Z^{-1}}$ with probability at least

\begin{align}
	1 - \mathop{\Pr}_{Z^{-1}}\left[\exists i \, F_{Z^{-1}}(q_i) \neq F_Z(q_i)\right] 
	&\ge 1 - 4NQ\sqrt{\frac{1-r}{2\pi Mr}} 2^{-M(1-H_2(r))}.\label{eq:query_concentration}
\end{align}

Now in order to apply Lemma \ref{lem:binomial} to bound $\epsilon$ as a function of $\delta$, we want to show for each $j$ that the event $f_j(x+\rho) \neq f_j(x)$ is independent of the query answers $a_1,\dotsc, a_Q$. This can be done by applying Lemma \ref{lem:trim} again to each column $j$ to show that with high probability, the query answers only depend on the random sampling of $Z^{-j}$. Since $\rho=\phi(\{a_k\}_{k=1}^Q)$ is a function of the query answers, then this means that the adversary's chosen $\rho$ also only depends on $Z^{-j}$. We obtain

\begin{align*}
	\mathop{\Pr}_{Z^j} \left[ f_j(x+\rho) \neq f_j(x) \, | \, a_1,\dotsc, a_Q\right] 
	&= \mathop{\Pr}_{Z^j} \left[ f_j(x+\rho) \neq f_j(x) \, | \, Z^{-j}\right]\\
	&= \mathop{\Pr}_{Z^j} \left[ f_j(x+\rho) \neq f_j(x)\right],
\end{align*}
and we see that this probability has no dependence on the actual column $j$ since $Z^j$ is independent and identical for each $j$. We incur a factor $M$ in the probability of failure by applying a union bound of the failure probability in (\ref{eq:query_concentration}) over all $j=1,\dotsc,M$. Thus the event $f_j(x+\rho)\neq f_j(x)$ is independent and identical for each column $j$ with probability at least

\begin{align*}
	1 - 4NQ\sqrt{\frac{M(1-r)}{2\pi r}} 2^{-M(1-H_2(r))}.
\end{align*}

Then by Lemma \ref{lem:binomial}, the probability of $\rho$ changing the output of $f_{\overline{z}}$ is at least

\begin{align*}
	1 - 2\left(r + \sqrt{\frac{\log(1/\epsilon)}{2M}}\right).
\end{align*}
\end{proof}

\section{Probability inequalities}

\begin{lemma}\label{lem:binom_coefficient_bound}
	Suppose $\lambda n$ is an integer, where $0 < \lambda < 1$. Then
	\begin{align*}
		\dbinom{n}{\lambda n} &\le \frac{1}{\sqrt{2\pi n\lambda (1-\lambda)}}2^{n H_2(\lambda)}
	\end{align*}
	where $H_2(\lambda) = -\lambda \log_2\lambda - (1-\lambda) \log_2(1-\lambda)$ is the negative entropy function.
\end{lemma}

\begin{lemma}\label{lem:hoeffding_binomial}[Hoeffding's inequality]
	Suppose $X \sim \text{Binom}(n, p)$. Then for any $\alpha>0$,
	\begin{align*}
		\Pr\left[X \le (p-\alpha)n\right] &\le \exp\left(-2\alpha^2 n\right)
	\end{align*}
\end{lemma}

\end{document}